\documentclass{article}
\usepackage{arxiv}

\usepackage{hyperref}
\usepackage{microtype}

\usepackage{amssymb}
\usepackage{amsmath}
\usepackage{amsthm}
\usepackage{mathdots}
\usepackage{mathtools}

\usepackage{tensor}

\usepackage{urwchancal}
\DeclareFontFamily{OT1}{pzc}{}
\DeclareFontShape{OT1}{pzc}{m}{it}{<-> s * [1.10] pzcmi7t}{}
\DeclareMathAlphabet{\mathpzc}{OT1}{pzc}{m}{it}

\usepackage{graphicx}
\usepackage{ifpdf}
\ifpdf
\usepackage{epstopdf}
\PrependGraphicsExtensions{.svg}
\fi

\usepackage{xypic}

\newtheorem{theorem}{Theorem}[section]
\newtheorem{lemma}[theorem]{Lemma}

\newtheorem{proposition}[theorem]{Proposition}

\providecommand{\R}{\mathbb{R}}

\providecommand{\SO}{\mathbf{SO}}

\providecommand{\SE}{\mathbf{SE}}
\providecommand{\SOT}{\mathbf{SOT}}

\providecommand{\MR}{\mathbf{MR}}

\providecommand{\grpG}{\mathbf{G}}

\providecommand{\gothg}{\mathfrak{g}}

\providecommand{\so}{\mathfrak{so}}
\providecommand{\se}{\mathfrak{se}}

\providecommand{\Sph}{\mathrm{S}}

\providecommand{\calM}{\mathcal{M}}
\providecommand{\calN}{\mathcal{N}}

\providecommand{\calU}{\mathcal{U}}

\providecommand{\Sym}{\mathbb{S}} 

\providecommand{\eb}{\mathbf{e}} 

\DeclareMathOperator{\Ad}{Ad}

\providecommand{\id}{\mathrm{id}} 

\providecommand{\tT}{\mathrm{T}} 
\providecommand{\td}{\mathrm{d}}
\providecommand{\tD}{\mathrm{D}}

\providecommand{\ddt}{\frac{\td}{\td t}}

\providecommand{\mr}[1]{{#1}^\circ} 

\usepackage{accents}
\makeatletter
\providecommand{\scirc}{%
    \hbox{\fontfamily{\rmdefault}\fontsize{0.4\dimexpr(\f@size pt)}{0}\selectfont{\raisebox{-0.52ex}[0ex][-0.52ex]{$\circ$}}}}

\makeatother

\mathchardef\mhyphen="2D

\providecommand{\etal}{\textit{et al.}~}

\usepackage[utf8]{inputenc}
\usepackage{amsmath}
\usepackage{amsfonts}
\usepackage{amssymb}
\usepackage[all]{xy}

\renewcommand{\mr}[1]{#1^\circ}

\usepackage{hyperref}
\providecommand{\tT}{\mathrm{T}}
\providecommand{\vinsG}{\mathbf{SLAM}^\text{\tiny VI}_n(3)}
\providecommand{\vinsg}{\mathfrak{slam}^\text{\tiny VI}_n(3)}
\providecommand{\vinsT}{\mathcal{T}_n^\text{\tiny VI}(3)}
\providecommand{\vinsM}{\mathcal{M}_n^\text{\tiny VI}(3)}
\providecommand{\vinsN}{\mathcal{N}_n^\text{\tiny  V}(3)}

\newcommand{\papercitation}{
van Goor, P., Mahony, R. (2020), ``An Equivariant Filter for Visual Inertial Odometry'', \emph{Accepted for presentation at IEEE ICRA}.
}

\newcommand{\publicationdetails}
{
\papercitation \\
\copyrightNoticeIEEE{2021} 
}
\newcommand{\publicationversion}
{Author accepted version}

\title{\LARGE \bf
An Equivariant Filter for Visual Inertial Odometry
}

\headertitle{An Equivariant Filter for Visual Inertial Odometry}

\author{
    \href{https://orcid.org/0000-0003-4391-7014}{\includegraphics[scale=0.06]{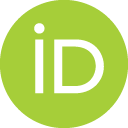}\hspace{1mm}
Pieter van Goor}
\\
    Systems Theory and Robotics Group \\
    Australian Centre for Robotic Vision \\
	Australian National University \\
    ACT, 2601, Australia \\
    \texttt{Pieter.vanGoor@anu.edu.au} \\
	\And	\href{https://orcid.org/0000-0002-7803-2868}{\includegraphics[scale=0.06]{orcid.png}\hspace{1mm}
    Robert Mahony}
\\
    Systems Theory and Robotics Group \\
    Australian Centre for Robotic Vision \\
	Australian National University \\
    ACT, 2601, Australia \\
	\texttt{Robert.Mahony@anu.edu.au} \\
}

\begin{document}





\maketitle
\thispagestyle{empty}
\pagestyle{empty}

\begin{abstract}
Visual Inertial Odometry (VIO) is of great interest due the ubiquity of devices equipped with both a monocular camera and Inertial Measurement Unit (IMU).
Methods based on the extended Kalman Filter remain popular in VIO due to their low memory requirements, CPU usage, and processing time when compared to optimisation-based methods.
In this paper, we analyse the VIO problem from a geometric perspective and propose a novel formulation on a smooth quotient manifold where the equivalence relationship is the well-known invariance of VIO to choice of reference frame.
We propose a novel Lie group that acts transitively on this manifold and is compatible with the visual measurements.
This structure allows for the application of Equivariant Filter (EqF) design leading to a novel filter for the VIO problem.
Combined with a very simple vision processing front-end, the proposed filter demonstrates state-of-the-art performance on the EuRoC dataset compared to other EKF-based VIO algorithms.
\end{abstract}

\section{Introduction}

Visual Inertial Odometry (VIO) belongs to the more general class of spatial awareness problems often referred to as Simultaneous Localisation and Mapping (SLAM).
SLAM algorithms are a core technology in mobile robotics and have been the subject of significant research for at least 30 years \cite{2006_durrant-whyte_SimultaneousLocalisation}.
The particular problem of Visual Inertial SLAM (VI-SLAM), where the only available sensors are an Inertial Measurement Unit (IMU) and a monocular camera continues to see substantial interest due the low-cost of the required sensors and the breadth of applications \cite{2018_delmerico_BenchmarkComparison}.
Visual inertial odometry and visual inertial SLAM share the same formulation, however, the odometry problem focuses on estimating the robot trajectory while the SLAM problem places equal emphasis on the map.
In practice, the difference is characterised by how long feature points are stored and whether full loop-closure is considered in the algorithms.
State-of-the-art solutions to VIO can be broadly classified into optimisation-based or Extended Kalman Filter (EKF)-based systems.
Optimisation-based solutions, including ORB-SLAM 3 \cite{2020_campos_ORBSLAM3Accurate}, OKVIS \cite{2015_leutenegger_KeyframebasedVisual} and VINS-Mono \cite{2018_qin_VinsmonoRobust}, treat VI-SLAM as a non-linear least squares problem and optimise over a moving window of data measurements.
In contrast, EKF-based solutions, such as ROVIO \cite{2015_bloesch_RobustVisual}, MSCKF \cite{2007_mourikis_MultistateConstraint} and SVO \cite{2016_forster_SVOSemidirect}, model the state estimate as a normal distribution, linearise the state equations and apply an EKF to the resulting error coordinates.
While optimisation methods typically achieve the highest accuracy in computing the robot's trajectory, EKF methods remain of interest due to their lower memory requirements and processing times \cite{2018_delmerico_BenchmarkComparison}.

\begin{figure}[!tb]
    \centering
    \includegraphics[width=0.8\linewidth]{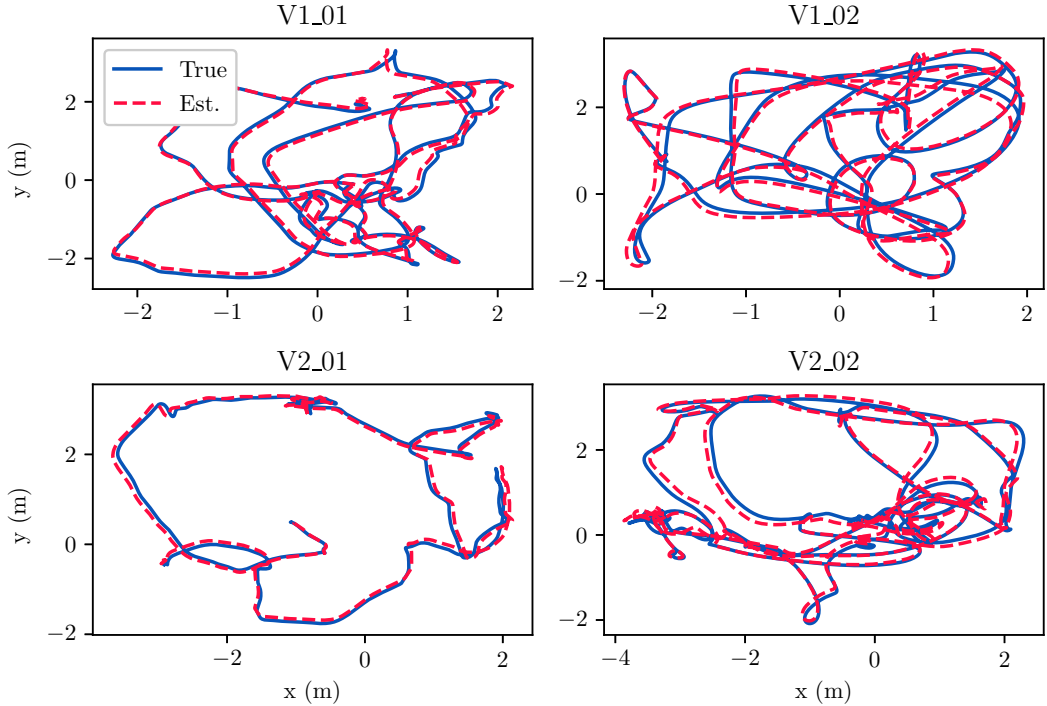}
    \caption{The true and estimated trajectory of the robot in the considered EuRoC sequences.}
    \label{fig:trajectories}
\end{figure}

Classical EKF designs for VIO are well-known to grow overconfident in their state over time \cite{2004_castellanos_LimitsConsistency,2006_bailey_ConsistencyEKFSLAM}.
This inconsistency is a direct consequence of unobservability of the visual inertial SLAM problem expressed in inertial coordinates \cite{2004_andrade-cetto_EffectsPartial, 2006_lee_ObservabilityObservability, 2010_huang_ObservabilitybasedRules}.
In \cite{2011_kelly_VisualInertialSensor}, Kelly and Sukhatme provide a complete characterisation of the non-linear observability of visual inertial SLAM, and show that the problem has a four-dimensional unobservable subspace corresponding to the position and yaw of the reference frame.
This issue can be mitigated by careful choice of linearisation points of the EKF to ensure the observability of the linearised system reflects that of the true system \cite{2010_huang_ObservabilitybasedRules}.
While this improves the consistency of the state estimate, it does not avoid the existence of unobservable subspaces in the state space and, as a consequence, the underlying Riccati equation may grow unbounded \cite{1968_wonham_MatrixRiccati}.
Recently, a number of authors have exploited the Invariant Extended Kalman Filter \cite{2015_barrau_InvariantExtended} and the novel Lie group structure proposed in \cite{2016_barrau_EKFSLAMAlgorithm} to develop filters for SLAM
\cite{2016_barrau_EKFSLAMAlgorithm,2017_wu_InvariantEKFVINS,2017_zhang_ConvergenceConsistency,2018_heo_ConsistentEKFbased}.
These filters still suffer from unobservable states, however, invariance properties are exploited to ensure consistency of the filter, and the unbounded growth in the covariance of the unobservable state \cite{1968_wonham_MatrixRiccati} can be managed using heuristics.
A further limitation of the symmetry is that it is not compatible with visual feature outputs, and although the state propagation linearisation is exact, the IEKF suffers from output linearisation error for the visual SLAM problem.
Mahony \etal \cite{2017_mahony_GeometricNonlinear} introduced a quotient manifold structure, termed the SLAM-manifold, that overcomes the observability issues by providing a fully observable state space for the SLAM problem that is geometrically motivated.
Van Goor \etal \cite{2019_vangoor_GeometricObserver} introduced a new symmetry for the SLAM problem that acts transitively on the SLAM-manifold, and in addition is compatible with visual point feature measurements, overcoming the limitation of the symmetry \cite{2016_barrau_EKFSLAMAlgorithm} associated with linearisation of the output function.
However, although this symmetry acts transitively on the SLAM-manifold \cite{2017_mahony_GeometricNonlinear}, the IEKF is only formulated for systems posed directly on a Lie-group and cannot be applied to systems on homogeneous spaces such as is the case for the SLAM-manifold formulation with either of the symmetries \cite{2016_barrau_EKFSLAMAlgorithm,2017_mahony_GeometricNonlinear} or \cite{2019_vangoor_GeometricObserver}.
In contrast, the recently proposed Equivariant Filter (EqF) \cite{2020_vangoor_EquivariantFilter} is explicitly posed for systems on general homogeneous spaces and can be applied.



In this paper, we derive a novel VIO filter based on equivariance principles that has state-of-the-art performance.
We show that there is a natural invariance in the traditional inertial coordinates of the visual inertial SLAM formulation associated with the gauge transformation that leads to the unobservability properties that are well known \cite{2011_kelly_VisualInertialSensor}.
We show that the quotient of the inertial coordinates by the gauge transform generates a smooth manifold we term the \emph{VI-SLAM manifold} on which the system is fully observable.
The symmetry group first proposed in \cite{2019_vangoor_GeometricObserver} is easily generalised to a new Lie-group, we term the \emph{VI-SLAM group},
that acts transitively on the \emph{VI-SLAM manifold} and is compatible with the vision measurements of points features.
This provides the geometric structure necessary to implement the Equivariant Filter (EqF) \cite{2020_vangoor_EquivariantFilter}.
The resulting algorithm benefits from all the advantages of a complete Lie group symmetry that is compatible with the measurements as well as being fully observable, overcoming limitations of previous invariant filter algorithms for visual inertial SLAM problems.
Finally, we demonstrate the performance of the proposed system on sequences in the EuRoC dataset and achieve state of the art results compared to EKF-based algorithms in spite of the simplicity of our front-end image processing system.

\section{Preliminaries}

For a background on smooth manifolds, Lie groups and their actions, the authors recommend \cite[Chapter 7]{2012_lee_SmoothManifolds}.
For a smooth manifold $\calM$, let $\tT_\xi \calM$ denote the tangent space of $\calM$ at $\xi$ and let $\tT\calM$ denote the tangent bundle.
Given a differentiable function between smooth manifolds $h:\calM \to \calN$, the linear map
\begin{align*}
    \tD_\xi |_{\xi'} h(\xi): \tT_{\xi'} \calM &\to  \tT_{h({\xi'})} \calN, \\
    v &\mapsto \tD_\xi |_{\xi'} h(\xi)[v],
\end{align*}
denotes the differential of $h$ with respect to the argument $\xi$ evaluated at $\xi'$.
The map
\begin{align*}
    \td h: \tT \calM &\to \tT \calN, \\
    (\xi',v) &\mapsto (h(\xi'), \tD_\xi |_{\xi'} h(\xi) [v]),
\end{align*}
denotes the differential of $h$ where the base point is implicit in the argument.
That is, given $ v \in \tT_{\xi'} \calM$ for some $\xi' \in \calM$, and a function $h: \calM \to \calN$, we write
\begin{align*}
    \td h [v] := \tD_\xi |_{\xi'} h(\xi) [v] \in \tT_{h(\xi')} \calN.
\end{align*}

We make extensive use of a number of Lie groups.
For a Lie group $\grpG$ we denote the Lie algebra $\gothg$.
The Special Orthogonal group $\SO(3)$ has elements $R \in \SO(3)$ and acts on $q \in \R^3$ by $R(q) = Rq$.
The Special Euclidean group $\SE(3)$ has elements $P = (R_P, x_P) \in \SO(3) \ltimes \R^3$ and acts on $q \in \R^3$ by $P(q) = R_P q + x_P$.
The Extended Special Euclidean group $\SE_2(3)$ has elements $(A, w) \in \SE(3) \times \R^3$ with group multiplication $(A_1, w_1) \cdot (A_2, w_2) = (A_1 A_2, w_1 + R_A w_2)$ \cite{2015_barrau_InvariantExtended}.
The positive multiplicative reals $\MR(1)$ has elements $c > 0$.
The Scaled Orthogonal Transformations $\SOT(3)$ has elements $Q = (R_Q, c_Q) \in \SO(3) \times \MR(1)$ and acts on $q \in \R^3$ by $Q(q) = c_Q R_Q q$ \cite{2019_vangoor_GeometricObserver}.

For any $\Omega \in \R^3$ define
\begin{align*}
    \Omega^\times := \begin{pmatrix}
                0 & -\Omega_3 &  \Omega_2 \\
         \Omega_3 &         0 & -\Omega_1 \\
        -\Omega_2 &  \Omega_1 &         0 \\
    \end{pmatrix}.
\end{align*}
Then the Lie algebra $\so(3) = \{ \Omega^\times \in \R^{3 \times 3} | \Omega \in \R^3 \}$, and
$$\Omega^\times p = \Omega \times p = -p \times \Omega = - p^\times \Omega,$$
for any $\Omega, p \in \R^3$, where $\times$ is the usual cross product.

For any $\Omega, v \in \R^3$ define $U(\Omega, v) \in \se(3)$ such that
for any $P = (R_P, x_P) \in \SE(3)$,
\begin{align*}
    \dot{P} = P U(\Omega, v)
    \; \Leftrightarrow \;
    \dot{R}_P = \dot{R} \Omega^\times \text{ and } \dot{x}_P = R_P v.
\end{align*}


A (right) group action of a Lie group $\grpG$ on a smooth manifold $\calM$ is a smooth function $\phi : \grpG \times \calM \to \calM$ satisfying
\begin{align*}
    \phi(XY, \xi) = \phi(Y, \phi(X, \xi)), \quad \phi(\id, \xi) = \xi.
\end{align*}
Given such a map, we denote by $\phi_X$ and $\phi_\xi$ the partial maps
\begin{align*}
    \phi_X &: \calM \to \calM, & \phi_X(\xi) &:= \phi(X, \xi), \\
    \phi_\xi &: \grpG \to \calM, & \phi_\xi(X) &:= \phi(X, \xi).
\end{align*}

Denote the 2-sphere by $\Sph^2 = \{y \in \R^3 \; \vline \; \Vert y \Vert = 1 \}$.
Let $\pi_{\Sph^2}(q) := q / \Vert q \Vert$ be the sphere projection for all $q \in \R^3$ with $q \neq 0$.
For $\eb_1 = (1,0,0)$ let $\vartheta_{\eb_1} : \calU_{\eb_1} \subset \Sph^2 \to \R^2$ be the stereographic projection as in \cite[Chapter 1]{2012_lee_SmoothManifolds}, and for every $\eta \in \Sph^2 \setminus \{\eb_1\}$, define $\vartheta_{\eta} : \calU_\eta \subset \Sph^2 \to \R^2$ by
\begin{align}
    \vartheta_\eta(y) := \vartheta_{\eb_1}( y - 2 \zeta \zeta^\top y  ), \quad \zeta := \pi_{\Sph^2}(\eb_1 - \eta), 
    \label{eq:StereoGraphic}
\end{align}
where $\calU_\eta = \Sph^2 \setminus \{-\eta\}$.
Then for each $\eta \in \Sph^2$, $\vartheta_\eta$ is a local coordinate chart with $\vartheta_\eta(\eta) = 0$.

Let 
\[
\Sym_+(n) = \{ S = S^\top \in \R^{n\times n} | x^\top S x > 0, \text{ for all } x \not= 0  \in \R^n \}
\]
denote the set of positive definite symmetric matrices of dimension of $n$. 

\section{Problem Description}

Choose an arbitrary inertial reference frame $\{0\}$, and consider a robot equipped with an IMU and a camera, both of which are rigidly attached.
For simplicity we identify the IMU frame $\{ I\}$ with the robot's body-fixed frame $\{ B\}$.
The inertial coordinates for the visual inertial SLAM problem are \begin{align}
    (P, v, p_1, ..., p_n) \in \SE(3) \times \R^3 \times (\R^3)^n, \label{eq:total_space_state}
\end{align}
where,
\begin{itemize}
    \item $P = (R_P, x_P) \in \SE(3)$ is the pose of the IMU $\{B\}$ with respect to the inertial frame $\{0\}$,
    \item $v \in \R^3$ is the linear velocity of the robot in the body-fixed frame $\{B\}$,
    \item $p_i \in \R^3$ is the coordinates of landmark $i$ in the inertial frame $\{0\}$.
\end{itemize}
We frequently use the notation $(P, v, p_i) \equiv (P, v, p_1,...,p_n)$ as shorthand.
To ensure that the visual measurements are always well defined we assume that the trajectory considered never passes through an exception set $\mathcal{E} \subset \SE(3) \times \R^3 \times (\R^3)^n$ corresponding to all situations where the camera centre coincides with a landmark point.
To formalise this, we define the visual inertial SLAM (VI-SLAM) total space
\[
\vinsT := \SE(3) \times \R^3 \times (\R^3)^n - \mathcal{E}
\]
and consider the visual inertial SLAM problem on $\vinsT$.
Note that $\vinsT$ is an open subset of a smooth manifold and as such is itself a smooth manifold.

Let the acceleration due to gravity in the inertial frame $\{ 0 \}$ be $g \eb_3$, where $g \approx 9.81$ m/s$^2$ and $\eb_3 \in \Sph^2$ is standard gravity direction in the inertial frame.
The ideal IMU measurements are $(\Omega, a) \in \R^3 \times \R^3$, the angular velocity and linear acceleration of the IMU, respectively.
The VIO system dynamics are
\begin{equation}
    \ddt (P, v, p_i) = f_{(\Omega, a)} (P, v, p_i), \label{eq:system_function}
\end{equation}
\vspace*{-0.5cm} 
\begin{align*}
    \dot{P} &= P U(\Omega, v), &
    \dot{v} &= -\Omega^\times v + a - g R_P^\top \eb_3, &
    \dot{p}_i &= 0.
\end{align*}

The camera measurements are modelled as $n$ bearing measurements of the landmarks $p_i$ in the camera frame $\{C\}$ on the manifold $\vinsN := (\Sph^2)^n$ where the superscript ``V'' stands for visual measurements.
Let $T_C \in \SE(3)$ denote the pose of the camera frame $\{C\}$ with respect the body frame $\{ B\}$.
We do not consider the online calibration of $T_C$ in the present work.
Then the measurement function $h : \vinsT \to \vinsN$ is given by
\begin{align}
    h(P, v, p_i) &:= \left( h^1(P, v, p_i), ... h^n(P, v, p_i) \right), \label{eq:measurement_function} \\
    h^k(P, v, p_i)) &:= \pi_{\Sph^2} \left( (P T_C)^{-1} (p_k) \right). \notag
\end{align}
Modelling the bearing measurements directly on the sphere rather than the image plane enables the proposed system to model a wide variety of monocular cameras.

\subsection{Invariance of Visual Inertial SLAM}

Let $\eb_3$ be the standard gravity direction and define the semi-direct product group
\begin{align*}
    \Sph^1 \ltimes_{\eb_3} \R^3 := \{
        (\theta, x) \; \vline \; \theta \in \Sph^1, x \in \R^3
    \},
\end{align*}
with group product, identity and inverse
\begin{align*}
    (\theta^1, x^1) \cdot (\theta^2, x^2) &= (\theta^1 + \theta^2, x^1 + R_{\eb_3}(\theta^1) x^2), \\
    \id_{\Sph^1 \ltimes_{\eb_3} \R^3} &= (0, 0_{3 \times 1}), \\
    (\theta, x)^{-1} &= (-\theta, - R_{\eb_3}(\theta) x ),
\end{align*}
where $R_{\eb_3}(\theta) \in SO(3)$ is the anti-clockwise rotation of an angle $\theta$ about the axis $\eb_3$.
Then $\Sph^1 \ltimes_{\eb_3} \R^3$ may be identified with the subgroup
\begin{align*}
    \SE_{\eb_3}(3) := \{
        (R, x) \in \SE(3) \; \vline \; R \eb_3 = \eb_3
    \} \leq \SE(3).
\end{align*}

Define $\alpha : \SE_{\eb_3} \times \vinsT \to \vinsT$ by
\begin{align*}
    \alpha(S, (P, v, p_i)) := (S^{-1} P, v, S^{-1}(p_i)).
\end{align*}
Then $\alpha$ is a (right) group action of $\SE_{\eb_3}(3)$ on $\vinsT$.
For a given $S \in \SE_{\eb_3}(3)$, the action $\alpha(S, \cdot)$ represents a change of inertial reference frame from $\{0\}$ to $\{1\}$ where $S$ is the pose of $\{1\}$ with respect to $\{0\}$.
Moreover, any change of reference $S \in \SE_{\eb_3}(3)$ leaves the direction of gravity $\eb_3$ unchanged.

\begin{proposition} \label{prop:invariance_action}
The system function \eqref{eq:system_function} and measurement function \eqref{eq:measurement_function} are invariant with respect to $\alpha$, that is,
\begin{align*}
    f_{(\Omega, a)} (\alpha(S, (P, v, p_i))) &= \td \alpha_S f_{\Omega, a} (P, v, p_i), \\
    h(\alpha(S, (P, v, p_i))) &= h(P, v, p_i),
\end{align*}
for any $S \in \SE_{\eb_3}(3)$.
\end{proposition}

A proof is provided in Appendix~\ref{app:proofs}.


\subsection{VI-SLAM Manifold}

We exploit the invariance of the dynamics and measurements to propose a new state space where the system is fully observable.
Given any $(P, v, p_i) \in \vinsT$, define the equivalence class
\begin{align*}
    [P, v, p_i] &= \{
        \alpha(S, (P, v, p_i)) \; \vline \; S \in \SE_{\eb_3}(3)
    \}.
\end{align*}
Since $\alpha$ is a proper group action the associated quotient is a smooth manifold that we term the \emph{Visual Inertial SLAM (VI-SLAM) manifold}
\begin{align*}
    \vinsM &:= \vinsT / \alpha
    = \{[P, v, p_i] \; \vline \;  (P, v, p_i) \in \vinsT \},
\end{align*}
with projection map $\pi(P,v,p_i) := [P, v, p_i]$.
The induced system and measurements functions on $\vinsM$ are well-defined due to their invariance with respect to $\alpha$.
Transformation by the subgroup $\SE_{\eb_3}(3)$ corresponds directly the unobservable states in the inertial SLAM coordinates \cite{2011_kelly_VisualInertialSensor}.
It follows that the SLAM problem posed on the VI-SLAM manifold is fully observable since the quotient operation factors out the unobservable states while preserving the observable information.


\section{Equivariant Filter for VI-SLAM}
\label{sec:eqf-design}

\subsection{Symmetry of VI-SLAM}

The Equivariant Filter (EqF) \cite{2020_vangoor_EquivariantFilter} exploits symmetries of systems on homogeneous spaces to design a filter about a fixed linearisation point on the state manifold with a constant output linearisation.

Let $\vinsG = \SE_2(3) \times \SOT(3)^n$ denote the \emph{VI-SLAM Group} \cite{2019_vangoor_GeometricObserver} with group product, identity and inverse given by
\begin{align*}
(A^1, w^1, Q^1_i) \cdot (A^2, w^2, Q^2_i) = (A^1 A^2, w^1 + R_{A^1} w^2, Q^1_i Q^2_i), \\
\id = (I_4, 0_{3 \times 1}, (I_3)_i), \quad
(A,w,Q_i)^{-1} = (A^{-1}, -R_A w, Q_i^{-1}).
\end{align*}

This Lie group is a symmetry group that acts transitively on $\vinsM$ and $\vinsN$ in a compatible manner that makes both the system function $f$ \eqref{eq:system_function} and the output function $h$ \eqref{eq:measurement_function} equivariant.
The following lemmas are proved in Appendix \ref{app:proofs}. 

\begin{lemma} \label{lem:state_action}
The map $\Phi : \vinsG \times \vinsT \to \vinsT$ defined by
\begin{align}
    \Phi & ((A,w,Q_i), (P, v, p_i)) \notag \\
    &:= (PA, R_A^\top(v - w), PA T_C Q_i^{-1} T_C^{-1} P^{-1} (p_i)), \label{eq:total_state_action}
\end{align}
is a transitive (right) group action.
Moreover, the induced action $\phi : \vinsG \times \vinsM \to \vinsM$, given by
\begin{align}
    \phi((A,w,Q_i), [P, v, p_i]) := [\Phi((A,w,Q_i), (P, v, p_i))], \label{eq:manifold_state_action}
\end{align}
is well-defined.
\end{lemma}

\begin{lemma}\label{lem:output_action}
The map $\rho : \vinsG \times \vinsN \to \vinsN$ defined by
\begin{align}
    \rho((A,w,Q_i), (\eta_i)) := (R_{Q_i}^\top \eta_i),
\end{align}
is a (right) group action.
Additionally, the measurement function \eqref{eq:measurement_function} is equivariant with respect to the actions $\phi$ \eqref{eq:manifold_state_action} and $\rho$, that is,
\begin{align*}
    h(\phi((A,w,Q_i), [P, v, p_i])) = \rho((A,w,Q_i), h([P, v, p_i])),
\end{align*}
for all $(A,w,Q_i) \in \vinsG$ and $[P, v, p_i] \in \vinsM$.
\end{lemma}

The existence of a transitive action by the VI-SLAM group on the VI-SLAM manifold guarantees the existence of an equivariant lift \cite{2020_mahony_EquivariantSystems}.
That is, the system dynamics may be lifted to the symmetry group.

\begin{lemma}\label{lem:lift}
    The map $\Lambda : \vinsT \times (\R^3 \times \R^3) \to \vinsg$, given by
    \begin{align}
        \Lambda & ((P, v, p_i), (\Omega, a)) \\
        &:= \left(
            U(\Omega, v), \; -a + g R_P^\top \eb_3, \; \left( \Omega_C + \frac{q_i^\times v_C}{\Vert q_i \Vert^2}, \frac{q_i^\top v_C}{\Vert q_i \Vert^2} \right)_i
        \right), \notag \\
        q_i &:= (P T_C)^{-1} (p_i), \qquad
        (\Omega_C, v_C) := \Ad_{T_C}^{-1} (\Omega, v), \notag
    \end{align}
    is a lift \cite{2020_mahony_EquivariantSystems} of the system function \eqref{eq:system_function}.
    That is,
    \begin{align*}
        \tD_{E} |_\id \phi_{(P, v, p_i)} (E) \Lambda((P, v, p_i), (\Omega, a)) = f_{(\Omega, a)} (P, v, p_i).
    \end{align*}
    Moreover, the induced map $\Lambda : \vinsM \times (\R^3 \times \R^3) \to \vinsg$ is well-defined and also a lift for the system on the VI-SLAM manifold.
\end{lemma}


\subsection{Origin Choice and Local Coordinates}

The EqF design procedure requires a choice of origin configuration and local coordinates.
Let $\mr{\Xi}  = (\mr{P}, \mr{v}, \mr{p}_i) \in \vinsT$ denote a fixed \emph{origin configuration} and set $\mr{\xi} = [\mr{\Xi}] = [\mr{P}, \mr{v}, \mr{p}_i] \in \vinsM$.
The filter state for the EqF is an element of the VI-SLAM group, $\hat{X} \in \vinsG$, and the associated state estimate is obtained by applying the group action $\hat{\Xi} =  \Phi(\hat{X}, \mr{\Xi})$ \eqref{eq:total_state_action} \cite{2020_mahony_EquivariantSystems}.

Choose the map $\varepsilon : \calU_{\mr{\xi}} \subset \vinsM \to \R^{5+3n}$ defined by
\begin{align} \label{eq:state_local_coords}
    \varepsilon([P, v, p_i])
    &:= \begin{pmatrix}
        \vartheta_{R_{\mr{P}}^\top \eb_3}(R_P^\top \eb_3) \\
        v - \mr{v} \\
        T_C^{-1}(P^{-1}(p_1) - {\mr{P}}^{-1}(\mr{p}_1)) \\
        \vdots \\
        T_C^{-1}(P^{-1}(p_n) - {\mr{P}}^{-1}(\mr{p}_n))
    \end{pmatrix},
\end{align}
to be the coordinate chart for $\vinsM$.
Let $(\mr{y}_i) = h(\mr{\xi})$.
Choose the map $\delta : \calU_{(\mr{y}_i)} \subset \vinsN \to \R^{2n}$ defined by
\begin{align} \label{eq:output_local_coords}
    \delta(y_1, ..., y_n)
    &:= \left( \vartheta_{\mr{y}_1}(y_1), ..., \vartheta_{\mr{y}_n}(y_n) \right),
\end{align}
to be the local coordinate chart for $\vinsN$ where $\vartheta$ is the stereographic projection of the sphere \eqref{eq:StereoGraphic}. 
Note that $\varepsilon(\mr{\xi}) = 0$ and $\delta(\mr{y}_i) = 0$.

\subsection{Input Bias}

We model real-world IMU measurements as having a constant (or slowly time-varying) bias,
\begin{align*}
    \Omega_m &= \Omega + b_\Omega, & a_m &= a + b_a, \\
    \dot{b}_\Omega &= 0, & \dot{b}_a &= 0,
\end{align*}
where $\Omega_m, a_m \in \R^3$ are the measured angular velocity and linear acceleration, respectively, and $b_\Omega, b_a \in \R^3$ are the biases.
Let $b = (b_\Omega, b_a) \in \R^6$.

\subsection{EqF with Bias Dynamics}

Let $\hat{X} \in \vinsG$ be the observer state \cite{2020_mahony_EquivariantSystems} and let $\hat{b} = (\hat{b}_\Omega, \hat{b}_a) \in \R^6$ be the estimated bias with dynamics
\begin{align*}
    \dot{\hat{X}} &= \hat{X} \Lambda( \phi(\hat{X}, \mr{\xi}), (\hat{\Omega}, \hat{a})) - \Delta \hat{X}, &
    \hat{X}(0) &= \id, \\
    \dot{\hat{b}} &= -\beta, &
    \hat{b}(0) &= 0,
\end{align*}
where $\hat{\Omega} = \Omega_m - \hat{b}_\Omega$, $\hat{a} = a_m - \hat{b}_a$ and $\Delta \in \vinsg$ and $\beta \in \R^6$ are correction terms.

Let $\mr{A}_t, B_t, \mr{C}$ be the EqF state, input, and output matrices, respectively, as described in Appendix \ref{app:matrices}.
Let $\Sigma \in \Sym_+(11+3n)$ be the Riccati term of the EqF with bias, with dynamics
\begin{align*}
    \dot{\Sigma}
    &= \begin{pmatrix}
        0 & 0 \\ -B_t & A_t
    \end{pmatrix} \Sigma + \Sigma \begin{pmatrix}
        0 & -B_t^\top \\ 0 & A_t^\top
    \end{pmatrix}
    + \begin{pmatrix}
        0 & 0 \\ 0 & B_t R_t B_t^\top
    \end{pmatrix}
    \\ &\hspace{0.5cm}
    + P_t
    - \Sigma \begin{pmatrix}
        0 & 0 \\ 0 & {\mr{C}}^\top Q_t^{-1} \mr{C}
    \end{pmatrix}
    \Sigma,
    \qquad
    \Sigma(0) = \Sigma_0
\end{align*}
where $\Sigma_0 \in \Sym_+(11+3n)$ is the initial Riccati term, and $P_t \in \Sym_+(11+3n)$, $R_t \in \Sym_+(6)$, and $Q_t \in \Sym_+(2n)$ are positive definite gain matrices.

Let $\xi = [P, v, p_i] \in \vinsM$ denote the true state of the system and let $e = \phi(\hat{X}^{-1}, \xi) \in \vinsM$ denote the global EqF error.
The correction term of the EqF with bias is determined by the lift of the Kalman update to the Lie-group.
That is
\begin{align}
    \Delta &:= \tD_E |_\id \phi_{\mr{\xi}}(E)^\dagger \cdot \tD_\xi |_{\mr{\xi}} \varepsilon(\xi)^{-1} \Gamma,
    \label{eq:eqf-innovation} \\
    \begin{pmatrix}
        \beta \\ \Gamma
    \end{pmatrix}
    &:= \Sigma
    \begin{pmatrix}
        0 & {\mr{C}}^\top
    \end{pmatrix}
    Q_t^{-1}
    \delta( \rho( \hat{X}^{-1}, h(\xi) ) ),
\end{align}
where $\tD_E |_\id \phi_{\mr{\xi}}(E)^\dagger$ is a suitably chosen right-inverse of $\tD_E |_\id \phi_{\mr{\xi}}(E)$.
Then the EqF state estimate is given by
\begin{align}
    (\hat{P}, \hat{v}, \hat{p}_i) := \Phi((\hat{A}, \hat{w}, \hat{Q}_i), (\mr{P}, \mr{v}, \mr{p}_i)). \label{eq:estimated-state}
\end{align}

\subsection{Bundle Lift}

The EqF is designed directly on the VI-SLAM manifold to overcome the unobservability of the problem.
However, in practice it is usual to report the state as an element of the total space.
Moreover, in lifting to the total space it is desirable to do so in such a manner to minimize the error introduced into the trajectory estimation.
The correction term is therefore lifted from the manifold to the total space by minimising the motion of landmark points with respect to the current choice of inertial frame, subject to a weighting term derived from the EqF Riccati matrix $\Sigma$.

Define the weighted cost function $J : \tT_{(\hat{P}, \hat{v}, \hat{p}_i)} \vinsT \to \R^+$ by
\begin{align*}
    J(\hat{P}U, u_v, u_{p_i}) :=
    \left\Vert \td \varepsilon \cdot \td \pi \cdot \td \Phi_{(\hat{A}, \hat{w}, \hat{Q}_i)}^{-1} (0, 0, u_{p_i}) \right\Vert^2_\Sigma,
\end{align*}
where $\Sigma$ is the EqF Riccati term, and $\Vert \cdot \Vert_\Sigma$ is the Mahalanobis norm.

Let $\Gamma \in \R^{5+3n}$ be the correction term defined in \eqref{eq:eqf-innovation}.
The correction on the total space $\Gamma' \in \tT_{(\mr{P}, \mr{v}, \mr{p}_i)} \vinsT$ is the solution of the optimisation problem
\begin{align*}
    \text{minimise } & J ( \tD_{\Xi} |_{(\mr{P}, \mr{v}, \mr{p}_i)} \Phi_{(\hat{A}, \hat{w}, \hat{Q}_i)}(\Xi) (\Gamma') ), \\
    \text{subject to } & \tD_\xi |_{[\mr{P}, \mr{v}, \mr{p}_i]} \varepsilon(\xi) \cdot \tD_\Xi |_{(\mr{P}, \mr{v}, \mr{p}_i)} \pi(\Xi) \Gamma' = \Gamma.
\end{align*}
This may be solved using weighted linear least-squares.

Finally, the correction term $\Delta \in \vinsg$ is chosen by
\begin{align*}
    \Delta = \tD_E |_\id \Phi_{(\mr{P}, \mr{v}, \mr{p}_i)}(E)^\dagger \Gamma',
\end{align*}
where $\tD_E |_\id \Phi_{(\mr{P}, \mr{v}, \mr{p}_i)}(E)^\dagger$ is an arbitrary fixed right-inverse of $\tD_E |_\id \Phi_{(\mr{P}, \mr{v}, \mr{p}_i)}(E)$.

\section{Experiments}

To demonstrate practical performance, we evaluated the proposed EqF on a number of sequences from the EuRoC dataset \cite{2016_burri_EuRoCMicro}.
Vision measurements were obtained by applying OpenCV functions \texttt{goodFeaturesToTrack} and \texttt{calcOpticalFlowPyrLK} to detect and track features.
The maximum number of features was kept to 50, and new features were detected whenever the number of features being tracked fell below 40.
The EqF gain matrices and parameters were kept consistent across all trials, and the observer dynamics were discretised using Euler integration.
The proposed system was implemented in \texttt{c++} and our code is available online\footnote{\url{https://github.com/pvangoor/eqf_vio}}.

We limited our attention to the easy and medium sequences in the Vicon rooms, as we found the Machine Hall and ``hard'' sequences to be too challenging for our vision processing front-end to track features reliably.
Figure \ref{fig:trajectories} shows the estimated trajectories compared with the ground truth.
Table \ref{tab:results} shows the Root Mean Square Error (RMSE) between ground truth trajectory of the robot and the estimated position reported by our system for each sequence.
The RMSE of popular EKF-based VI-SLAM solutions SVO-MSF, MSCKF, and ROVIO (obtained from \cite{2018_delmerico_BenchmarkComparison}), and the invariance based R-UKF-LG (obtained from \cite{2018_brossard_InvariantKalman}) are shown for comparison.
Due to significant differences in system architectures, the tuning parameters for each system cannot be compared directly, and this contributes to the differences in outcomes in Table \ref{tab:results}.

\begin{table}[!htb]
    \centering
    \begin{tabular}{c|ccccc}
        & \multicolumn{4}{c}{RMSE (m)} \\
        & R-UKF-LG\footnotemark & SVO-MSF & MSCKF  & ROVIO & EqF \\
        & \cite{2018_brossard_InvariantKalman} & \cite{2016_forster_SVOSemidirect} & \cite{2007_mourikis_MultistateConstraint} & \cite{2015_bloesch_RobustVisual} & * \\
        \hline
        V1 01 & 0.55 & 0.40 & 0.34 &          0.10 & \textbf{0.07} \\
        V1 02 & 0.40 & 0.63 & 0.20 & \textbf{0.10} &          0.11 \\
        V2 01 & 0.37 & 0.20 & 0.10 &          0.12 & \textbf{0.08} \\
        V2 02 & 0.47 & 0.37 & 0.16 &          0.14 & \textbf{0.13} \\
        \hline
    \end{tabular}
    \caption{Comparison of RMSE on the EuRoC dataset.}
    \label{tab:results}
\end{table}
\footnotetext{The R-UKF-LG \cite{2018_brossard_InvariantKalman} uses both the left and right cameras for stereo vision rather than monocular vision.}


The proposed EqF clearly outperforms competitor EKF-based algorithms shown in Table \ref{tab:results}.
ROVIO \cite{2015_bloesch_RobustVisual} achieves the lowest RMSE on \texttt{V1\_02}, likely thanks to the tight-coupling between the vision front-end and the filter back-end.
Figure \ref{fig:error_distribution} shows the distribution of the error of the trajectory estimated by our system for each sequence.

\begin{figure}[!htb]
    \centering
    \includegraphics[trim=0cm 0.5cm 0cm 1cm, width=0.6\linewidth]{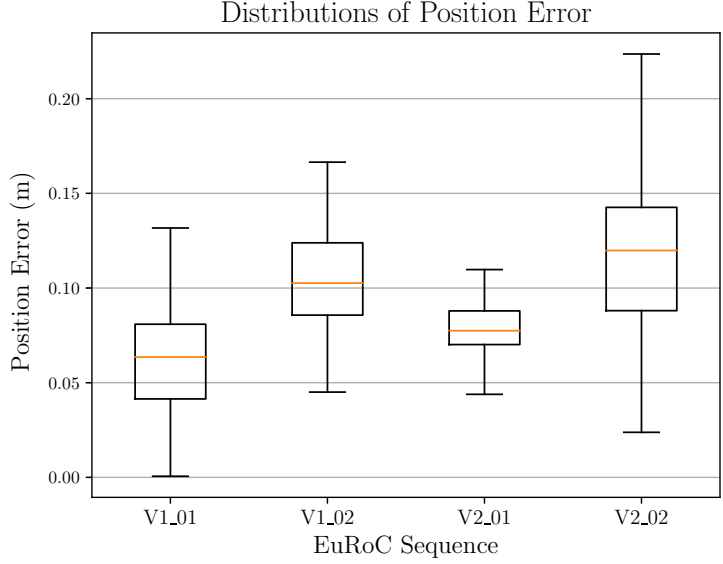}
    \caption{The distributions of position error of the estimated trajectories. }
    \label{fig:error_distribution}
\end{figure}

The pie chart in Figure \ref{fig:system_breakdown} shows the mean processing time per frame of each component of the full system.
The processing times were recorded on a desktop computer with an Intel\textregistered  Core\texttrademark  i7-8700 CPU @ 3.20GHz $\times$ 12 and 16GB of RAM.
The total processing time per frame was recorded at 5.4ms or 183.7Hz.
Of the total time, the vision front-end consumes 4.2ms and the filter consumes 1.2ms.
This high speed combined with the high accuracy reported in Table \ref{tab:results} clearly demonstrate the real-world potential of the proposed system.

\begin{figure}
    \centering
    \includegraphics[trim=0cm 1.5cm 0cm 0.5cm, clip=true, width=0.6\linewidth]{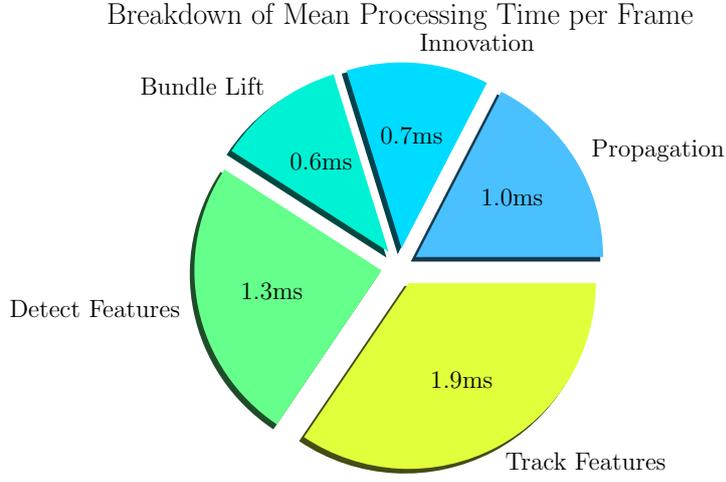}
    \caption{The mean processing time per frame for each component of the proposed system.}
    \label{fig:system_breakdown}
\end{figure}

\section{Conclusion}

This paper makes the following contributions.
\begin{itemize}
    \item The VI-SLAM manifold is constructed as a state space for visual inertial SLAM where the unobservability of the system associated with a change of reference frame is factored out.
    \item The VI-SLAM group is developed and shown to act equivariantly on the VI-SLAM total space, manifold and output space.
    \item An Equivariant Filter (EqF) is designed according to \cite{2020_vangoor_EquivariantFilter}, and coupled with a simple vision processing front-end to achieve state-of-the-art results on the EuRoC MAV dataset \cite{2016_burri_EuRoCMicro}.
\end{itemize}


\appendix{Proofs}
\label{app:proofs}


\begin{proof}[Proof of Proposition \ref{prop:invariance_action}]
Choosing $S \in \SE_{\eb_3}(3)$, $(P, v, p_i) \in \vinsT$ and $(\Omega, a) \in \R^3 \times \R^3$ arbitrary, one has
\begin{align*}
    f_{(\Omega, a)} (\alpha(S, & (P, v, p_i)))
    = f_{(\Omega, a)} (S^{-1} P, v, S^{-1}(p_i)), \\
    &= \left(
        S^{-1} P U, \; -\Omega^\times v + a - g R_P^\top R_S \eb_3, \; 0
    \right), \\
    &= \left(
        S^{-1} (P U), \; -\Omega^\times v + a - g R_P^\top \eb_3, \; 0
    \right), \\
    &= \td \alpha_S f_{\Omega, a} (P, v, p_i),
\end{align*}
as required.
Similarly, for any $k$, the partial measurement function $h^k$ \eqref{eq:measurement_function} satisfies
\begin{align*}
    h^k(\alpha(S, (P, v, p_i)))
    &= h^k (S^{-1} P, v, S^{-1}(p_i)), \\
    &= \pi_{\Sph^2} \left( (S^{-1} P T_C )^{-1} S^{-1} (p_k) \right), \\
    &= \pi_{\Sph^2} \left( (P T_C )^{-1} (p_k) \right), \\
    &= h^k(S, (P, v, p_i)),
\end{align*}
and the invariance of the full measurement function $h$ follows immediately.
\end{proof}

\begin{proof}[Proof of Lemma \ref{lem:state_action}]
The proof that $\Phi$ is a group action closely follows that of \cite[Lemma 4.2]{2019_vangoor_GeometricObserver}, and has been omitted from this paper to save space.
To see that $\phi$ is well-defined, observe that
\begin{align*}
    &\hspace{-0.1cm}\phi ( (A,w,Q_i), [S^{-1} P, v, S^{-1}(p_i)]) \\
    &= [\Phi((A,w,Q_i), (S^{-1} P, v, S^{-1}(p_i)))], \\
    &= [S^{-1} P A, R_A^\top(v - w), S^{-1} PA T_C Q_i^{-1} T_C^{-1} P^{-1} S S^{-1} (p_i)], \\
    &= [P A, R_A^\top(v - w), PA T_C Q_i^{-1} T_C^{-1} P^{-1} (p_i)], \\
    &= [\Phi((A,w,Q_i), (P, v, p_i))], \\
    &= \phi( (A,w,Q_i), [P, v, (p_i)]),
\end{align*}
as required.
\end{proof}

The proofs of Lemmas \ref{lem:output_action} and \ref{lem:lift} closely follow proofs previously published in \cite{2019_vangoor_GeometricObserver} and have been omitted from this paper to save space.

\appendix{EqF Matrices}
\label{app:matrices}

Here we present the state, input, and output matrices of the EqF proposed in Section \ref{sec:eqf-design}.
Let $(\mr{P}, \mr{v}, \mr{p}_i) \in \vinsT$ denote the origin coordinates, let $(\hat{A}, \hat{w}, \hat{Q}_i) \in \vinsG$ denote the observer state, let $(\hat{P}, \hat{v}, \hat{p}_i)$ denote the estimated state as in \eqref{eq:estimated-state}, and let the input to the system be $(\Omega, a) \in \R^3 \times \R^3$.
Define $(\hat{\Omega}_C, \hat{v}_C) := \Ad_{T_C}^{-1} (\Omega, \hat{v})$, $\hat{q}_i := (\hat{P} T_C)^{-1} (\hat{p}_i)$.

The EqF state matrix $\mr{A}_t$ is given by \cite[Lemma A.1]{2020_vangoor_EquivariantFilter},
\begin{align*}
    \mr{A}_t
    &= \begin{pmatrix}
        0 & 0 & 0 & \cdots & 0\\
        -g \tD_{z} |_{0} \vartheta_{\eb_3}^{-1}(z) & 0 & 0 & \cdots & 0 \\
        0 & - \hat{Q}_1 R_{\hat{A} T_{C}}^\top &  A_{\hat{q}_1} & 0 & 0 \\
        \vdots & \vdots & 0 & \ddots &  0 \\
        0 & - \hat{Q}_n R_{\hat{A} T_{C}}^\top & 0 & 0 & A_{\hat{q}_n}
    \end{pmatrix},
\end{align*}
where,
\begin{align*}
    A_{\hat{q}_i}
    &:= - \Vert \hat{q}_i \Vert^{-2}
    \hat{Q}_i (\hat{q}_i^\times v_C^\times
    - 2 v_C \hat{q}_i^\top
    + \hat{q}_i v_C^\top
    ) \hat{Q}_i^{-1} .
\end{align*}

The EqF input matrix $B_t$ is given by
\begin{align*}
    B_t &=
    \tD_\xi |_{[\mr{P}, \mr{v}, \mr{p}_i]} \varepsilon(\xi)
    \cdot \tD_\xi |_{[\hat{P},\hat{v},\hat{q}_i]} \phi_{(\hat{A},\hat{w},\hat{Q}_i)^{-1}}(\xi)
    \\ &\hspace{1.5cm}
    \cdot \tD_E |_\id \phi_{[\hat{P},\hat{v},\hat{q}_i]}(E)
    \cdot \tD_{u} |_{(\Omega, a)} \Lambda([\hat{P},\hat{v},\hat{q}_i], u), \\
    &= \begin{pmatrix}
        \tD_{\eta} |_{R_{\mr{P}}^\top \eb_3} \vartheta_{R_{\mr{P}}^\top \eb_3 }(\eta) R_{\hat{A}} (R_{\hat{A}}^\top \eb_3)^\times & 0 \\
        R_{\hat{A}} \hat{v}^\times & R_{\hat{A}} \\
        \hat{Q}_1 (\hat{q}_1^\times R_{T_C}^\top + R_{T_C}^\top x_{T_C}^\times) & 0 \\
        \vdots & \vdots \\
        \hat{Q}_n (\hat{q}_n^\times R_{T_C}^\top + R_{T_C}^\top x_{T_C}^\times) & 0
    \end{pmatrix}.
\end{align*}

Let $\mr{q}_i = (\mr{P} T_C)^{-1} (\mr{p}_i)$ and $(\mr{y}_i) = h(\mr{P}, \mr{v}, \mr{p}_i)$.
Then the (constant) EqF output matrix $\mr{C}$ is given by
\begin{align*}
    \mr{C}
    &= \begin{pmatrix}
        0 & 0 &  \mr{C}_1 & 0 & \cdots & 0 \\
        \vdots & \vdots & 0 & \ddots &  & \vdots \\
        \vdots & \vdots & \vdots & & \ddots & 0 \\
        0 & 0 & 0 & \cdots & 0 & \mr{C}_n
    \end{pmatrix},
\end{align*}
where each $\mr{C}_i \in \R^{2 \times 3}$ is given by
\begin{align*}
    \mr{C}_i = \tD_{\eta} |_{\mr{y}_i} \vartheta_{\mr{y}_i}(\eta) \frac{1}{\Vert \mr{q} \Vert} \left( I_3 - \frac{{\mr{q}} {\mr{q}}^\top}{{\mr{q}}^\top {\mr{q}}} \right).
\end{align*}


\bibliographystyle{plain}
\bibliography{equivariant_vio}

\end{document}